\documentclass{custom} 

\title[Rejection sampling in sublinear time]{Rejection sampling from shape-constrained distributions in sublinear time}

\usepackage{enumitem}
\usepackage{graphicx}
\usepackage{mathtools}
\usepackage{microtype}
\usepackage{times}
\usepackage[Symbolsmallscale]{upgreek}
\usepackage{thm-restate} 

\usepackage{hyperref}       
\usepackage{url}            
\usepackage{booktabs}       
\usepackage{amsfonts}       
\usepackage{nicefrac}       
\usepackage{microtype}      

\usepackage{algorithm}
\usepackage[noend]{algpseudocode}
\usepackage{amsmath, amssymb}
\usepackage{bbm}
\usepackage{mathrsfs}
\usepackage{mathtools}
\usepackage{ragged2e}
\usepackage{thmtools}
\usepackage{thm-restate}
\usepackage{tikz}
\usepackage{wrapfig}
\usepackage{style}
\usepackage{euscript}

\DeclareMathOperator{\unif}{unif}

\declaretheorem[name=Task, style=definition]{task}

\makeatletter
\def\set@curr@file#1{\def\@curr@file{#1}} 
\makeatother
\usepackage[load-configurations=version-1]{siunitx} 

\coltauthor{\Name{Sinho Chewi} \Email{schewi@mit.edu}\\
\Name{Patrik Gerber} \Email{prgerber@mit.edu}\\
 \Name{Chen Lu} \Email{chenl819@mit.edu}\\
 \Name{Thibaut {Le Gouic}} \Email{tlegouic@mit.edu}\\
 \Name{Philippe Rigollet} \Email{rigollet@mit.edu}\\
 \addr MIT}

\begin{document}

\maketitle

\begin{abstract}
    We consider the task of generating exact samples from a target distribution, known up to normalization, over a finite alphabet.
    The classical algorithm for this task is rejection sampling, and although it has been used in practice for decades, there is surprisingly little study of its fundamental limitations.
    In this work, we study the query complexity of rejection sampling in a minimax framework for various classes of discrete distributions. Our results provide new algorithms for sampling whose complexity scales sublinearly with the alphabet size.
    When applied to adversarial bandits, we show that a slight modification of the \EXPthree\ algorithm reduces the per-iteration complexity from $\mc O(K)$ to $\mc O(\log^2 K)$, where $K$ is the number of arms.
\end{abstract}

\section{Introduction}

Efficiently generating exact samples from a given target distribution, known up to normalization, has been a fundamental problem since the early days of algorithm design~\citep{marsaglia1963generating, walker1974new, kronmal1979alias, bratley2011guide, knuth2014art}. It is a basic building block of randomized algorithms and simulation, and understanding its theoretical limits is of intellectual and practical merit. Formally, let $p$ be a probability distribution on the set of integers $[N]:=\{1, \ldots, N\}$, and assume we are given query access to $\tilde p := Z p$, with an unknown constant $Z$. The Alias algorithm \citep{walker1974new} takes $\mathcal{O}(N)$ preprocessing time, after which one can repeatedly sample from $p$ in constant expected time. More recently, sophisticated algorithms have been devised to allow time-varying $\tilde p$ \citep{hagerup1993optimal, matias2003dynamic}, which also require $\mathcal{O}(N)$ preprocessing time. Unsurprisingly, for arbitrary $\tilde{p}$, the $\mc O(N)$ time is the best one can hope for, as shown in \cite{bringmann2017efficient} via a reduction to searching arrays.

A common element of the aforementioned algorithms is the powerful idea of rejection. Rejection sampling, along with Monte Carlo simulation and importance sampling, can be traced back to the work of Stan Ulam and John von Neumann \citep{vonNeumann1951, Ulam}. As the name suggests, rejection sampling is an algorithm which proposes candidate samples, which are then accepted with a probability carefully chosen to ensure that accepted samples have distribution $p$. Despite the fundamental importance of rejection sampling in the applied sciences, there is surprisingly little work exploring its theoretical limits. In this work, we adapt the minimax perspective, which has become a staple of the modern optimization~\citep{nesterov2018lectures} and statistics~\citep{tsybakov2009nonparametric} literature, and we seek to characterize the number of queries needed to obtain rejection sampling algorithms with constant acceptance probability (e.g.\ at least $1/2$), uniformly over natural classes of target distributions.

We consider various classes of shape-constrained discrete distributions that exploit the ordering of the set $[N]$ (monotone, strictly unimodal, discrete log-concave). We also consider a class of distributions on the complete binary tree of size $N$, were only a partial ordering of the alphabet is required. For each of these classes, we show that the rejection sampling complexity scales sublinearly in the alphabet size $N$, which can be compared with the literature on sublinear algorithms~\citep{Gol10, Gol17}. This body of work is largely focused on statistical questions such as estimation or testing and the present paper extends it in another statistical direction, namely sampling from a distribution known only up to normalizing constant, which is a standard step of Bayesian inference.

To illustrate the practicality of our methods, we present an application to adversarial bandits \citep{bubeckcesabianchi2012bandits} and describe a variant of the classical \EXPthree\ algorithm whose iteration complexity scales as $\mathcal{O}(\log^2(K))$ where $K$ is the number of arms.

\section{Background on rejection sampling complexity}\label{scn:rej_sampling}

\subsection{Classical setting with exact density queries}

To illustrate the idea of rejection sampling, we first consider the classical setting where we can make queries to the \emph{exact} target distribution $p$. 
Given a \emph{proposal distribution} $q$ and an upper bound $M$ on the ratio $\max_{x\in [N]} p(x)/q(x)$, rejection sampling proceeds by
drawing a sample $X \sim q$ and a uniform random variable $U \sim \unif(0,1)$. If $U \leq p(X)/(Mq(X))$, the sample $X$ is returned; otherwise, the whole process is repeated. Note that the rejection step is equivalent to flipping a biased coin: conditionally on $X$, the sample $X$ is accepted with probability $p(X)/(Mq(X))$ and rejected otherwise. We refer to this procedure as \emph{rejection sampling with acceptance probability $p/(Mq)$}.

It is easy to check that the output of this algorithm is indeed distributed according to $p$. Since $Mq$ forms an upper bound on $p$, the region $G_q = \{(x, y)\,:\,x\in [N],\,y \in [0,Mq(x)]\}$ is a superset of $G_p = \{(x,y)\,:\,x\in [N],\,y \in [0,p(x)]\}$. Then, a uniformly random point from $G_q$ conditioned on lying in $G_p$ is in turn uniform on $G_p$, and so its $x$-coordinate has distribution $p$. A good rejection sampling scheme hinges on the design of a good proposal $q$ that leads to few rejections.

If $q=p$, then the first sample $X$ is accepted. More generally, 
the number of iterations required before a variable is accepted follows a geometric distribution with parameter $1/M$ (and thus has expectation $M$). In other words, the bound $M$ characterizes the quality of the rejection sampling proposal $q$, and the task of designing an efficient rejection sampling algorithm is equivalent to determining a strategy for building the proposal $q$ which guarantees a small value of the ratio $M$ using few queries.

\subsection{Density queries up to normalization}\label{scn:density_up_to_normalization}

In this paper, we instead work in the setting where we can only query the target distribution up to normalization, which is natural for Bayesian statistics, randomized algorithms, and online learning. Formally, let $\mc P$ be a class of probability distributions over a finite alphabet $\ms X$, and consider a target distribution $p\in \mc P$. We assume that the algorithm $\mc A$ has access to an oracle which, given $x \in \ms X$, outputs the value $Zp(x)$, where $Z$ is an unknown constant. The value of $Z$ does not change between queries. Equivalently, we can think of the oracle as returning the value $p(x)/p(x_0)$, where $x_0 \in \ms X$ is a fixed point with $p(x_0) > 0$. 

To implement rejection sampling in this query model, the algorithm must construct an \emph{upper envelope} for $\tilde p$, i.e., a function $\tilde q$ satisfying $\tilde q \ge \tilde p$. We can then normalize $\tilde q$ to obtain a probability distribution $q$. To draw new samples from $p$, we first draw samples $X \sim q$, which are then accepted with probability $\tilde p(X)/\tilde q(X)$. The following theorem shows that the well-known guarantees for rejection sampling also extend to our query model. The proof is provided in Appendix~\ref{scn:proof_of_rej_sampling}.

\begin{theorem}\label{thm:rejection_sampling}
    Suppose we have query access to the unnormalized target $\tilde{p} = p Z_p$ supported on $\ms X$, and that we have an upper envelope $\tilde q \ge \tilde p$.
    Let $q$ denote the corresponding normalized probability distribution, and let $Z_q$ denote the normalizing constant, i.e., $\tilde q = q Z_q$.
    Then, rejection sampling with acceptance probability $\tilde{p}/\tilde q$ outputs a point distributed according to $p$, and the number of samples drawn from $q$ until a sample is accepted follows a geometric distribution with mean $Z_q/Z_p$.
\end{theorem}

After $n$ queries to the oracle for $p$ (up to normalization), the output $\mc A(n, \tilde p)$ of the algorithm is an upper envelope $\tilde q \ge \tilde p$, and in light of the above theorem it is natural to define the \emph{ratio}
\begin{align*}
    r(\mc A, n, \tilde p)
    &:= \frac{Z_q}{Z_p}
    = \frac{\sum_{x\in \ms X} \tilde q(x)}{\sum_{x\in \ms X} \tilde p(x)} \,.
\end{align*}
The ratio achieved by the algorithm determines the expected number of queries to $\tilde p$ needed to generate each new additional sample from $p$.

As discussed in the introduction, our goal when designing a rejection sampling algorithm is to minimize this ratio uniformly over the choice of target $p\in \mc P$.
We therefore define the rejection sampling complexity of the class $\mc P$ as follows.

\begin{definition}\label{defn:complexity}
   For a class of distributions $\mc P$, \emph{the rejection sampling complexity of $\mc P$} is the minimum number $n\in \N$ of queries needed, such that there exists and algorithm $\mc A$ that satisfies
   \begin{align*}
       \sup_{\tilde p \in \tilde{\mc P}} r(\mc A, n,\tilde p) \le 2\,,
   \end{align*}
   where $\tilde{\mc P} := \{ \tilde p= Zp\,:\, Z >0\}$ is the set of all positive rescalings of distributions in $\mc P$.
\end{definition}
The constant $2$ in Definition~\ref{defn:complexity} is arbitrary and could be replaced by any number strictly greater than $1$, but we fix this choice at $2$ for simplicity.
With this choice of constant, and once the upper envelope is constructed, new samples from the target can be generated with a constant ($\leq 2$) expected number of queries per sample.

Note that when the alphabet $\ms X$ is finite and of size $N$, then $N$ is a trivial upper bound for the complexity of $\mc P$, simply by querying all of the values of $\tilde p$ and then returning the exact upper envelope $\mc A(N, \tilde p) = \tilde p$. Therefore, for the discrete setting, our interest lies in exhibiting natural classes of distributions whose complexity scales \emph{sublinearly} in $N$.

In this work, we specifically focus on \emph{deterministic} algorithms $\mc A$. In fact, we believe that adding internal randomness to the algorithm does not significantly reduce the query complexity.
Using Yao's minimax principle~\citep{yao1977minimax}, it seems likely that our lower bounds can extended to hold for randomized algorithms.
We leave this extension for future work.

\section{Results for shape-constrained discrete distributions}

In order to improve on the trivial rate of $\mc O(N)$ on an alphabet of size $N$, we need to assume some structure of the target distributions. A well-known set of structural assumptions are shape constraints \citep{groeneboom2014nonparametric, silvapulle2011constrained}, which have been extensively studied in the setting of estimation and inference. When the alphabet is $[N]$, shape constraints are built on top of the linear ordering of the support. We show that such assumptions indeed significantly reduce the complexity of the restricted classes of distributions to sublinear rates. We also consider the setting where the linear ordering of the support is relaxed to a partial ordering, and show it also results in sublinear complexity

Our complexity results for various classes of discrete distributions are summarized in~\autoref{tab:discrete}. We define the various classes below, and give the sublinear complexity algorithms that construct the upper envelopes in Figure~\ref{fig:algs}.

\begin{table}
  \caption{Rejection sampling complexities for classes of discrete distributions. Here, $N$ always denotes the alphabet size, $\ms X = \{1,\dotsc,N\}$.}
  \label{tab:discrete}
  \centering
  \begin{tabular}{ccccc}
    \toprule
    Class     & Definition     & Complexity & Theorem & Algorithm \\
    \midrule
    monotone & \autoref{defn:monotone}  & $\Theta(\log N)$  & \autoref{thm:monotone} & Algorithm~\ref{algo:monotone_upper}  \\
    strictly unimodal     & \autoref{defn:unimodal} & $\Theta(\log N)$  & \autoref{thm:unimodal} & Algorithm~\ref{alg:unimodal} \\
    cliff-like     & \autoref{defn:cliff}  & $\Theta(\log \log N)$ & \autoref{thm:cliff} & Algorithm~\ref{alg:log_concave} \\
    discrete log-concave     & \autoref{defn:lc}  & $\Theta(\log \log N)$ & \autoref{thm:discrete_lc} & Algorithm~\ref{alg:log_concave}\\
    monotone on a binary tree & \autoref{defn:tree} & $\Theta(N/\log N)$ & \autoref{thm:tree} & Algorithm~\ref{alg:tree} \\
    \bottomrule
  \end{tabular}
\end{table}

\begin{figure}
  \begin{algorithm}[H]
  \caption{Construct upper envelope for monotone distributions on $[N]$}\label{algo:monotone_upper}
  \begin{algorithmic}[1]
  \State Query the values $\tilde p(2^i)$, $0 \le i\le \lceil \log_2 N \rceil - 1$.
  \State Construct the upper envelope $\tilde q$ as follows: set $\tilde q(1) := \tilde p(1)$, and
  \begin{align*}
      \tilde q(x) := \tilde p(2^i) \, , \qquad \text{ for } x\in (2^i, 2^{i+1}]\,.
  \end{align*}
  \end{algorithmic}
\end{algorithm}
  \begin{algorithm}[H]
  \caption{Construct upper envelope for strictly unimodal distributions on $[N]$}\label{alg:unimodal}
  \begin{algorithmic}[1]
  \State Use binary search to find the mode of $\tilde p$.
  \State Use Algorithm~\ref{algo:monotone_upper} to construct an upper envelope on each side of the mode.
  \end{algorithmic}
\end{algorithm}
  \begin{algorithm}[H]
  \caption{Construct upper envelope for discrete log-concave distributions on $[N]$}\label{alg:log_concave}
  \begin{algorithmic}[1]
  \State Use binary search to find the first index $1 \le i \le \lceil \log_2 N \rceil$ such that $\tilde p(2^i) \le \tilde p(1)/2$, or else determine that $i$ does not exist.
  \State If $i$ does not exist, output the constant upper envelope $\tilde q \equiv \tilde p(1)$.
  \State Otherwise, output
  \begin{align*}
      \tilde q(x)
      := \begin{dcases} \tilde p(1)\,, & x < 2^{i}\,, \\ \tilde p(2^i) \exp\bigl[- \frac{\log(\tilde p(1)/\tilde p(2^i))}{2^i - 1} \, (x - 2^i) \bigr]\,, & x \ge 2^i\,. \end{dcases}
  \end{align*}
  \end{algorithmic}
\end{algorithm}
  \begin{algorithm}[H]
  \caption{Construct upper envelope for monotone distributions on binary trees of size $[N]$}\label{alg:tree}
  \begin{algorithmic}[1]
  \State Query $\tilde p(x)$ for all vertices $x$ which are at depth at most $\ell_0 := \ell - \lfloor \log_2\ell \rfloor + 1$, where $\ell$ is the maximum depth of the tree.
  \State Output
\begin{align*}
    \tilde q(x)
    &:= \begin{cases} \tilde p(x)\,, & \text{if}~\on{depth}(x) \le \ell_0\,, \\ \tilde p(y)\,, & \text{if}~\on{depth}(x) > \ell_0\,,\; \on{depth}(y) = \ell_0\,,\; \text{and}~x~\text{is a descendant of}~y\,. \end{cases}
\end{align*}
  \end{algorithmic}
\end{algorithm}
\caption{Algorithms for constructing rejection sampling upper envelopes which attain the minimax rates described in Table~\ref{tab:discrete}.}
\label{fig:algs}
\end{figure}

\subsection{Structured distributions on a linearly ordered set}

A natural class of discrete distributions which exploits the linear ordering of the set $[N]$ is the class of monotone distributions, defined below.

\begin{definition}\label{defn:monotone}
    The class of \emph{monotone} distributions on $[N]$ is the class of probability distributions $p$ on $[N]$ with $p(1) \ge p(2) \ge p(3) \ge \cdots \ge p(N)$.
\end{definition}



We show in~\autoref{thm:monotone} that the rejection sampling complexity of the class of monotone distributions is $\Theta(\log N)$, achieved via Algorithm~\ref{algo:monotone_upper}. It is also straightforward to extend Algorithm~\ref{algo:monotone_upper} to handle the class of strictly unimodal distributions defined next (see~\autoref{thm:unimodal} and Algorithm~\ref{alg:unimodal}).

\begin{definition}\label{defn:unimodal}
    The class of \emph{strictly unimodal} 
    distributions on $[N]$ is the class of probability distributions $p$ on $[N]$ such that: there exists a point $x \in [N]$ with $p(1) < p(2) < \cdots < p(x)$ and $p(x) > p(x+1) > \cdots > p(N)$.
\end{definition}

It is natural to ask whether further structural properties can yield even faster algorithms for sampling. This is indeed the case, and we start by illustrating this on a simple toy class of distributions.

\begin{definition}\label{defn:cliff}
    The class of \emph{cliff-like} distributions on $[N]$ is the class of probability distributions $\unif([N_0])$ for $N_0 \in [N]$.
\end{definition}

Since the class of cliff-like distributions is contained in the class of monotone distributions, Algorithm~\ref{algo:monotone_upper} yields a simple upper bound of $\mc O(\log N)$ for this class. However, we can do better by observing that in order to construct a good rejection sampling upper envelope for this class, we do not need to locate the index $N_0$ of the cliff exactly; it suffices to find it approximately, which in this context means finding an index $N_0'$ such that $N_0' \le N_0 \le 2N_0'$. Since we only need to search over $\mc O(\log N)$ possible values for $N_0'$, binary search can accomplish this using only $\mc O(\log \log N)$ queries. We prove in~\autoref{thm:cliff} that this rate is tight.

\begin{remark}
The class of cliff-like distributions provides a simple example of a class for which obtaining queries to the exact distribution is \emph{not} equivalent to obtaining queries for the distribution up to a normalizing constant.
Indeed, in the former model, the value of $p(1) = 1/N_0$ reveals the distribution in one query, implying a complexity of $\Theta(1)$, whereas we prove in~\autoref{thm:cliff} that the complexity under the second model is $\Theta(\log \log N)$.
\end{remark}

Instead of formally describing the algorithm for sampling from cliff-like distributions, we generalize the algorithm to cover a larger class of structured distributions: the class of \emph{discrete log-concave distributions} \cite[see \S 4]{saumard2014log}.


\begin{definition}\label{defn:lc}
    The class of \emph{discrete log-concave} distributions on $[N]$ is the class of probability distributions $p$ on $[N]$ such that for all $x \in \{2,\dotsc,N-1\}$, we have ${p(x)}^2 \ge p(x-1) p(x+1)$. Equivalently it is the class of distributions $p$ on $[N]$ for which there exists a convex function $V : \R\to\R \cup \{\infty\}$ such that $p(x) = \exp(-V(x))$ for all $x \in [N]$. In addition, we assume that the common mode of all of the distributions is at $1$.\footnote{Without this condition, the class of discrete log-concave distributions includes the family of all Dirac measures on $[N]$, and the rejection sampling complexity is then trivially $\Theta(N)$.}
\end{definition}

We prove in~\autoref{thm:discrete_lc} that the rejection sampling complexity of discrete log-concave distributions is $\Theta(\log \log N)$, achieved by Algorithm~\ref{alg:log_concave} (note that this algorithm also applies for cliff-like distributions, since cliff-like distributions are discrete log-concave).

\begin{remark}
The class of discrete log-concave distributions is another case for which rejection sampling with exact density queries is much easier than with queries up to a normalizing constant. In the former model, \cite{devroye1987simple} requires only a single query to construct a rejection sampling upper envelope with ratio $\le 5$. In contrast, we show in~\autoref{thm:discrete_lc} that the complexity under the second model is $\Theta(\log\log N)$.
\end{remark}

\subsection{Monotone on a binary tree}

The previous examples of structured classes all rely on the linear ordering of $[N]$. We now show that it is possible to develop sublinear algorithms when the linear ordering is relaxed to a partial ordering.
Specifically, we consider a structured class of distributions on balanced binary trees (note that the previously considered distributions can be viewed as distributions on a path graph).

\begin{definition}\label{defn:tree}
   The class of \emph{monotone distributions on a binary tree} with $N$ vertices is the class of probability distributions $p$ on a binary tree with $N$ vertices, with maximum depth $\lceil \log_2(N+1)\rceil$, such that for every non-leaf vertex $x$ with children $x_1$ and $x_2$, one has $p(x) \ge p(x_1) + p(x_2)$.
\end{definition}

We prove in~\autoref{thm:tree} that the rejection sampling complexity of this class is $\Theta(N/\log N)$; the corresponding algorithm is given as Algorithm~\ref{alg:tree}.

In a sense, \autoref{defn:tree} reduces to the class of monotone distributions when the underlying graph is a path, since each vertex in the (rooted) path graph has one ``child''.The reader may wonder whether replacing the condition $p(x) \ge p(x_1) + p(x_2)$ with $p(x) \ge p(x_1) \vee p(x_2)$ is more natural. In~\autoref{thm:alternate_monotone}, we show that rejection sampling cannot achieve sublinear complexity under the latter definition.

\section{Application to bandits}

Rejection sampling does not just provide us with a method for sampling from a target distribution; it provides us with the stronger guarantee of an upper envelope $\tilde q \ge \tilde p$, with a bound on the ratio of the normalizing constants of $\tilde q$ and $\tilde p$ (see Section~\ref{scn:density_up_to_normalization}). In this section, we show how this stronger property can be used to provide a faster, approximate implementation of the anytime variant of the \EXPthree\ algorithm. We expect that rejection sampling can yield similar computational speedups while retaining performance guarantees for other randomized algorithms.

Recall the adversarial bandit problem~\cite[Ch.\ 3]{bubeckcesabianchi2012bandits}: given $K$ arms, at each step $t \in [T]$ the player chooses an arm $I_t \in [K]$ to play. Simultaneously, an adversary chooses a loss vector $\ell_t \in [0,1]^K$. The chosen arm is then played, and the player incurs a loss of $\ell_t(I_t)$. The aim of the player is to find a strategy that minimizes the pseudo-regret, defined by
\begin{equation*}
    \overline{R}_n = \E \sum\limits_{t=1}^T \ell_t(I_t) - \min\limits_{k \in [K]} \E \sum\limits_{t=1}^T \ell_t(k)\,.
\end{equation*}
See Algorithm \ref{algo:EXP3} for the strategy known as  \EXPthree, which achieves a pseudo-regret of at most $2\sqrt{TK\log K}$ \cite[Theorem 3.1]{bubeckcesabianchi2012bandits}, which is minimax optimal up to the factor of $\sqrt{\log K}$~\cite[Theorem 3.4]{bubeckcesabianchi2012bandits}. In what follows $x(i)$ denotes the $i$'th coordinate of a vector $x$, and $e_j$ denotes the $j$'th standard basis vector in $\R^K$.
\begin{wrapfigure}[14]{L}{0.49\textwidth}
\vspace{.2cm}
\begin{minipage}{0.49\textwidth}
\begin{algorithm}[H]
  \caption{The  \EXPthree\ algorithm.}\label{algo:EXP3}
  \begin{algorithmic}[1]
    \Procedure{\EXPthree}{$T$, ${(\eta_t)}_{t=1}^T$}
    \State set $L_0 := 0$ and $p_0 := \unif\{1,\dots,K\}$
    \For{$t = 1,\dotsc, T$}
\State draw and play $I_t \sim p_{t-1}$\label{EXP3:draw}
\State observe loss $\ell_t(I_t)$
\State set $L_t := L_{t-1} + e_{I_t} \ell_t(I_t) / p_{t-1}(I_t)$\label{EXP3:loss update}
\State set $p_t \propto \exp(-\eta_t L_t)$\label{EXP3:prob update}
\EndFor
    \EndProcedure
  \end{algorithmic}
\end{algorithm}
\end{minipage}
\end{wrapfigure}

The computationally intensive steps of the iteration in Algorithm \ref{algo:EXP3} are drawing the sample on line \ref{EXP3:draw} and updating the distribution on line \ref{EXP3:prob update}. For each $t$, let us write $\tilde{p}_t = \exp(-\eta_t L_t)$ for the unnormalized version of $p_t$. Note that $\tilde{p}_t$ is fully determined by $L_t$ and $\eta_t$. Thus, if we can sample from $\tilde{p}_{t-1}$ on line \ref{EXP3:draw} in $o(K)$ time, then we can improve the na\"{\i}ve iteration complexity of $\Theta(K)$ since we can just skip line \ref{EXP3:prob update}. We achieve this by constructing a specialised data structure $\mathcal{D}$ that maintains the empirical loss vector $L$ in sorted order, thereby allowing fast sampling via Algorithm~\ref{algo:monotone_upper}. We record the requirements on $\mathcal{D}$ in the lemma below. 
\begin{lemma}\label{lem:data structure}
There exists a data structure $\mathcal{D}$ that stores a length-$K$ array $L$ and supports the following operations in $\mathcal{O}(\log K)$ worst-case time:
\begin{enumerate}
\item Given $(L[i], i)$ and a number $\ell$, set $L[i] = \ell$. \label{data structure op2}
\item Given $k \in [K]$, output the $k$-th largest element of the array $(L[k],k)_{k \in [K]}$ (in the dictionary order). \label{data structure op3}
\end{enumerate}
\end{lemma}

For a proof of~\autoref{lem:data structure}, see Section~\ref{sec:data structure proof}. Let us now describe a minor modification of the  \EXPthree\ algorithm which has (virtually) identical performance guarantees with per-iteration complexity $\mathcal{O}(\log^2 K)$. First, instead of sampling from $p_{t-1}$ directly, we do so using the rejection sampling proposal $q_{t-1}$ constructed from $\tilde p_{t-1}$ via our algorithm for monotone distributions (Algorithm~\ref{algo:monotone_upper}); this is possible because~\autoref{lem:data structure} gives us query access to the sorted version of $L_t$. Second, we modify the unbiased estimator of the loss in line~\ref{EXP3:loss update} accordingly.

The new unbiased estimator of the loss is defined as follows.
Draw another independent arm $J \sim q_t$ and replace line~\ref{EXP3:loss update} with
\begin{equation*}
    L_t := L_{t-1} + \frac{e_{I_t} \ell_t(I_t)}{\tilde{p}_{t-1}(I_t)} \, \frac{\tilde{p}_{t-1}(J)}{q_{t-1}(J)}\,.
\end{equation*}
Observe that if
$\mathcal{F}_t$ denotes the $\sigma$-algebra generated by all rounds up to time $t-1$ as well as the randomness of the adversary in step $t$, then
\begin{align*}
    \E\Bigl[ \frac{e_{I_t} \ell_t(I_t)}{\tilde{p}_{t-1}(I_t)} \, \frac{\tilde{p}_{t-1}(J)}{q_{t-1}(J)} \Bigm\vert \mathcal{F}_t\Bigr]
    &= \E\Bigl[ \frac{e_{I_t} \ell_t(I_t)}{p_{t-1}(I_t)} \Bigm\vert \mathcal{F}_t\Bigr] \E\Bigl[ \frac{p_{t-1}(J)}{q_{t-1}(J)} \Bigm\vert \mathcal{F}_t\Bigr]
    = \E\Bigl[ \frac{e_{I_t} \ell_t(I_t)}{p_{t-1}(I_t)} \Bigm\vert \mathcal{F}_t\Bigr]
    = \ell_t\,.
\end{align*}
The modified algorithm is given as Algorithm~\ref{algo:EXP3 mod}.

\begin{algorithm}[H]
  \caption{The modified version of the  \EXPthree\ algorithm.}\label{algo:EXP3 mod}
  \begin{algorithmic}[1]
    \Procedure{Fast-\EXPthree}{$T$, ${(\eta_t)}_{t=0}^{T-1}$}
    \State set $L_0 := 0$ and $p_0 := \unif\{1,\dots,K\}$
    \For{$t = 1,\dotsc, T$}
    \State build rejection sampling proposal $q_{t-1}$ of $\tilde{p}_{t-1}\label{EXP3 mod:envelope} :=\exp(-\eta_{t-1}L_{t-1})$
    \State draw and play $I_t \sim p_{t-1}$ via rejection sampling using $q_{t-1}$\label{EXP3 mod:sample}
    \State draw $J \sim q_{t-1}$ independently\label{EXP3 mod: sample2}
\State observe loss $\ell_t(I_t)$
\State set $L_t := L_{t-1} + e_{I_t} \ell_t(I_t) \tilde{p}_{t-1}(J)/ (\tilde{p}_{t-1}(I_t) q_{t-1}(J))$\;\label{EXP3 mod:loss update}
\EndFor
    \EndProcedure
  \end{algorithmic}
\end{algorithm}

Let us verify the claimed $\mc O(\log^2 K)$ per-iteration complexity of Algorithm~\ref{algo:EXP3 mod}. Let $\mathcal{D}$ be an instance of the data structure described in~\autoref{lem:data structure} and initialize it with $L = L_0$. Building the rejection envelope $q_{t-1}$ on line \ref{EXP3 mod:envelope} requires $\mc O(\log K)$ calls to operation \ref{data structure op3}, for a total complexity of $\mathcal{O}(\log^2 K)$. Sampling from $q_{t-1}$ on line \ref{EXP3 mod:sample} requires $\mathcal{O}(\log K)$ time and performing a rejection step takes one call to operation \ref{data structure op3}, so the expected complexity of this step is $\mathcal{O}(\log K)$. Drawing $J$ on line \ref{EXP3 mod: sample2} requires $\mc O(\log K)$ time and one call to operation \ref{data structure op3}, and finally, line \ref{EXP3 mod:loss update} requires one call to operation \ref{data structure op2}.

The following result (proven in Appendix~\ref{scn:proof_regret}) provides a pseudo-regret guarantee.

\begin{proposition}\label{prop:EXP3 mod performance}
Algorithm \ref{algo:EXP3 mod} with step size $\eta_t := \frac 12 \sqrt{\frac{\log K}{K(t+1)}}$ satisfies
\begin{equation*}
    \overline{R}_n \leq 4\sqrt{TK\log K}\,. 
\end{equation*}
\end{proposition}
We regard the above result as a proof of concept for the use of rejection sampling to more efficiently implement subroutines in randomized algorithms.
Our proof of~\autoref{prop:EXP3 mod performance} follows well-known arguments from the bandit literature, and the key new ingredient is our strong control of the ratio between the target and proposal distribution, which allows us to bound the variance of our unbiased estimator of the loss.

\begin{remark}\label{rem:exp3 reducing variance}
In Algorithm \ref{algo:EXP3 mod} one may replace $\tilde{p}_{t-1}(J)/q_{t-1}(J)$ by $(1/m) \sum_{i=1}^m \tilde{p}_{t-1}(J_i)/q_{t-1}(J_i)$ for i.i.d.\ $J_i \sim q_{t-1}$ in order to further reduce the variance of the estimator. 
\end{remark}

\begin{figure}
\centering
\begin{minipage}{0.47\textwidth}
  \centering
  \includegraphics[width=\linewidth]{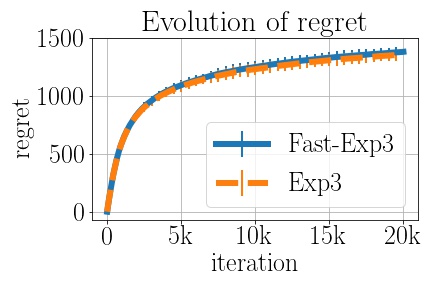}
  \caption{Error bars denote $4$ standard deviations over $20$ runs.}
  \label{fig:changing-cliff regret}
\end{minipage}%
\hfill
\begin{minipage}{0.47\textwidth}
  \centering
  \includegraphics[width=\linewidth]{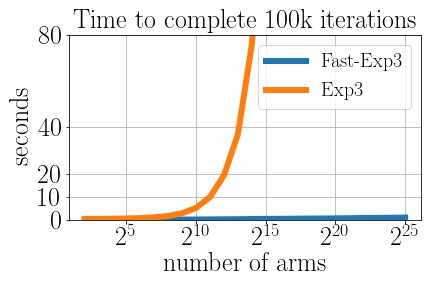}
  \caption{Comparison of iteration speed.}
  \label{fig:EXP3 speed main}
\end{minipage}
\end{figure}

We conduct a small simulation study to confirm that Algorithm \ref{algo:EXP3 mod} is competitive with \EXPthree. In ~\autoref{fig:changing-cliff regret} we plot the regret of the two algorithms over $T=20\text{k}$ steps using the theoretical step size $\eta_t = \sqrt{\log K / (K\,(t+1))}$ and $m=5$ (see~\autoref{rem:exp3 reducing variance}). We run the algorithms on a toy problem with $K=256$ arms, where $10\%$ of the arms always return a loss of $0$, and the remaining arms always return the maximal loss of $1$. In~\autoref{fig:EXP3 speed main} we compare the time it takes for the two algorithms to complete an iteration on the same problem, but with varying number of arms $K$. In the case of \EXPthree,\, when $\log_2 K > 20$ the values are not computed, and for $\log_2 K \leq 20$ they are extrapolated from $100$ iterations, as the running time becomes prohibitive for such large experiments. 

\begin{remark}
We note that for constant step size $\eta_t \equiv \eta$ (e.g.\ when the time horizon $T$ is known in advance), theoretically it is possible to implement \EXPthree\, with constant iteration cost by applying the results of \cite{matias2003dynamic}. Further, using the `doubling trick' such an algorithm can be turned into an anytime algorithm by restarting it $\approx \log T$ times at a computational cost of $\mathcal{O}(K)$ each time. However, our Algorithm \ref{algo:EXP3 mod} is the first to implement the more elegant solution of decaying step size in sublinear time per iteration.
\end{remark}

\section{Conclusion and outlook}

We studied the query complexity of rejection sampling within a minimax framework, and we showed that for various natural classes of discrete distributions, rejection sampling can obtain exact samples with an expected number of queries which is sublinear in the size of the support of the distribution. Our algorithms can also be run in sublinear time, which make them substantially faster than the baseline of multinomial sampling, as shown in our application to the  \EXPthree\ algorithm.

A natural direction for future work is to investigate the complexity of rejection sampling on other structured classes of distributions, such as distributions on graphs, or distributions on continuous spaces. In many of these other settings, the complexity of algorithms based on Markov chains has been studied extensively, but the complexity of rejection sampling remains to be understood.

\section*{Acknowledgments}

Sinho Chewi was supported by the Department of Defense (DoD) through the National Defense Science \& Engineering Graduate Fellowship (NDSEG) Program. Thibaut Le Gouic was supported by NSF award IIS-1838071. Philippe Rigollet was supported by NSF awards IIS-1838071, DMS-1712596,  and DMS-2022448.

\bibliography{ref.bib}

\appendix

\section{Guarantee for rejection sampling}\label{scn:proof_of_rej_sampling}

\begin{proof}[Proof of Theorem~\ref{thm:rejection_sampling}]
Since $\tilde q$ is an upper envelope for $\tilde p$, then $\tilde p(X)/\tilde q(X) \le 1$ is a valid acceptance probability.
Clearly, the number of rejections follows a geometric distribution. The probability of accepting a sample is given by
\begin{align*}
    \Pr(\text{accept})
    &= \int_{\ms X} \frac{\tilde p(x)}{\tilde q(x)} \, q(\D x)
    = \frac{Z_p}{Z_q} \int_{\ms X} p(\D x)
    = \frac{Z_p}{Z_q}\,.
\end{align*}
Let $X_1, X_2, X_3\dots$ be a sequence of i.i.d. samples from $q$ and let $U_1, U_2, U_3\dots$ be i.i.d. $\unif[0,1]$. Let $A \subseteq \ms X$ be a measurable set, and let $X$ be the output of the rejection sampling algorithm. Partitioning by the number of rejections, we may write
\begin{align*}
    \Pr(X \in A) &= \sum\limits_{n=0}^\infty \Pr\Bigl(X_{n+1} \in A,\; U_i > \frac{\tilde{p}(X_i)}{\tilde{q}(X_i)}\, \forall\,i\in [n],\; U_{n+1} \leq \frac{\tilde{p}(X_{n+1})}{\tilde{q}(X_{n+1})}\Bigr) \\
    &= \sum\limits_{n=0}^\infty \Pr\Bigl(X_{n+1} \in A,\, U_{n+1} \leq \frac{\tilde{p}(X_{n+1})}{\tilde{q}(X_{n+1})}\Bigr) \, \Pr\Bigl(U_1 > \frac{\tilde{p}(X_1)}{\tilde{q}(X_1)}\Bigr)^n \\
    &= \sum\limits_{n=0}^\infty \Bigl(\int_A \frac{\tilde{p}(x)}{\tilde{q}(x)} \, q(\D x)\Bigr) \Bigl(\int_{\ms X} \bigl(1 - \frac{\tilde{p}(x)}{ \tilde{q}(x)}\bigr) \, q(\D x)\Bigr)^n \\
    &= p(A) \, \frac{Z_p}{Z_q} \sum\limits_{n=0}^\infty \bigl(1 - \frac{Z_p}{Z_q}\bigr)^n
    = p(A)\,.
\end{align*}
\end{proof}

\section{Details for the bandit application}

\subsection{Pseudo-regret guarantee}\label{scn:proof_regret}

The proof below follows standard arguments in the bandit literature, e.g.~\cite[Theorem 3.1]{bubeckcesabianchi2012bandits}.
\begin{proof}[Proof of~\autoref{prop:EXP3 mod performance}]

    For $\eta > 0$, define the potential
\begin{equation*}
    \Phi_t(\eta) = \frac 1\eta \log \frac 1K \sum\limits_{i=1}^K \exp\bigl(-\eta L_t(i)\bigr)\,.
\end{equation*}
It is not difficult to verify that $\Phi_t'(\eta) \geq 0$ (see e.g. \cite[Proof of Theorem 3.1]{bubeckcesabianchi2012bandits}). Note additionally that $\Phi_0 \equiv 0$ and
\begin{equation*}
    \Phi_T(\eta) \geq -\min\limits_{i^\star \in [K]} L_T(i^\star) - \frac{\log K}{\eta}\,. 
\end{equation*}
For convenience, let $\eta_{-1} = \eta_0$. We get the chain of inequalities
\begin{align}
    \min\limits_{i^\star \in [K]} L_T(i^\star) + \frac{\log K}{\eta_{T-1}} &\geq \Phi_0(\eta_{-1}) - \Phi_T(\eta_{T-1}) \nonumber\\
    &= \sum\limits_{t=0}^{T-1} \left\{\Phi_t(\eta_{t-1}) - \Phi_{t+1}(\eta_t)\right\} \nonumber\\
    &\geq \sum\limits_{t=0}^{T-1} \left\{\Phi_t(\eta_t) - \Phi_{t+1}(\eta_t)\right\}, \label{eqn:regret telescope}
\end{align}
where the last inequality uses that $\eta_t \leq \eta_{t-1}$. The change in the potential from step $t$ to $t+1$ is
\begin{align*}
    \Phi_t(\eta_t) - \Phi_{t+1}(\eta_t) &= -\frac{1}{\eta_t} \log\frac{\sum_{i=1}^K \exp(-\eta_t L_{t+1}(i))}{\sum_{i=1}^K \exp(-\eta_t L_t(i))} \\
    &= -\frac{1}{\eta_t} \log\frac{\sum_{i=1}^K \exp(-\eta_t L_t(i)-\eta_t \one\{I_{t+1}=i\} \,\frac{\ell_{t+1}(i)}{\tilde{p}_t(i)}\,\frac{\tilde{p}_t(J)}{q_t(J)})}{\sum_{i=1}^K \exp(-\eta_t L_t(i))} \\
    &= -\frac{1}{\eta_t} \log \E_{i \sim p_t} \exp\Bigl(-\eta_t \one\{I_{t+1}=i\}\,\frac{\ell_{t+1}(i)}{\tilde{p}_t(i)} \, \frac{\tilde{p}_t(J)}{q_t(J)} \Bigr)\,.
\end{align*}
Using now that $e^{-x} \leq 1 - x + x^2/2$ for all $x \geq 0$ we write
\begin{align*}
    &\Phi_t(\eta_t) - \Phi_{t+1}(\eta_t) \\
    &\qquad \geq -\frac{1}{\eta_t} \log \E_{i \sim p_t}\Bigl[1 -\eta_t \one_{\{I_{t+1}=i\}}\, \frac{ \ell_{t+1}(i)}{\tilde{p}_t(i)}\,\frac{\tilde{p}_t(J)}{q_t(J)} + \frac{\eta_t^2}{2}  \one_{\{I_{t+1}=i\}} \, \Bigl(\frac{ \ell_{t+1}(i)}{\tilde{p}_t(i)} \,\frac{\tilde{p}_t(J)}{q_t(J)}\Bigr)^2\Bigr]\,. \\
    \intertext{Since $\log(1-x) \leq -x$, we further have}
    &\qquad \geq \sum\limits_{i=1}^K p_t(i) \one_{\{I_{t+1}=i\}} \,\frac{\ell_{t+1}(i)}{\tilde{p}_t(i)} \,\frac{\tilde{p}_t(J)}{q_t(J)} - \frac{\eta_t}{2}\sum\limits_{i=1}^K p_t(i) \one_{\{I_{t+1}=i\}} \,\Bigl( \frac{\ell_{t+1}(i)}{\tilde{p}_t(i)}\, \frac{\tilde{p}_t(J)}{q_t(J)}\Bigr)^2\,.
\end{align*}
Now, we take the expectation on both sides to obtain
\begin{equation*}
    \E[\Phi_t(\eta_t) - \Phi_{t+1}(\eta_t)] \geq \E\langle p_t, \ell_{t+1}\rangle - \frac{\eta_t}{2} \sum\limits_{i=1}^K \E\Bigl[p_t(i)^2\, \Bigl(\frac{\ell_t(i)}{p_t(i)}\, \frac{p_t(J)}{q_t(J)}\Bigr)^2\Bigr]\,,
\end{equation*}
where we used that $\tilde{p}_t(J)/\tilde{p}_t(i) = p_t(J)/p_t(i)$. The rejection sampling guarantee ensures that $\norm{p_t/q_t}_\infty \leq 2$ (see~\eqref{eq:easier_task}). This implies
\begin{align*}
    \E[\Phi_t(\eta_t) - \Phi_{t+1}(\eta_t)] &\geq \E\langle p_t, \ell_{t+1}\rangle - 2\eta_t \sum\limits_{i=1}^K \E[\ell_t(i)^2]
    \ge \E\langle p_t,\ell_{t+1}\rangle - 2\eta_t K\,.
\end{align*}
Plugging this bound into \eqref{eqn:regret telescope} we obtain
\begin{align*}
    \min\limits_{i^\star \in [K]}L_t(i^\star) + \frac{\log K}{\eta_{T-1}} &\geq \E \sum\limits_{t=0}^{T-1} \langle p_t, \ell_{t+1} \rangle - 2K\sum\limits_{t=0}^{T-1} \eta_t\,. 
\end{align*}
Rearranging yields the pseudo-regret guarantee
\begin{align*}
    \max\limits_{i^\star \in [K]} \E\Bigl[ \sum\limits_{t=1}^T \ell_t(I_t) - \sum\limits_{t=1}^T \ell_t(i^\star)\Bigr] \le \frac{\log K}{\eta_{T-1}} + 2 K \sum\limits_{t=0}^{T-1}\eta_t\,.
\end{align*}
Setting $\eta_t = \frac 12 \sqrt{\frac{\log K}{K(t+1)}}$ yields the bound $4\sqrt{TK\log K}$. 
\end{proof}

\subsection{The data structure} \label{sec:data structure proof}

\begin{proof}[Proof of~\autoref{lem:data structure}]
The data structure $\mathcal{D}$ is a self-balancing binary search tree with $K$ nodes, each of which contains the size of its left subtree as extra information. It is well-known that implementations of such a structure exist which support $\mathcal{O}(\log K)$ worst-case deletion, insertion, update, and search. In addition, it also supports finding the $k$-th largest element in $\mc O(\log K)$ time thanks to the extra information about the sizes of the subtrees. The data structure $\mathcal{D}$ is used to maintain the array $(L[k],k)_{k \in [K]}$ in sorted order (sorted according to the dictionary order). 
\end{proof}

\subsection{Experiments}
In addition to the experiments in the main text, we compare the performance of \EXPthree\, and Algorithm \ref{algo:EXP3 mod} on $2$ additional problems. Once again we run for $T=20\text{k}$ steps on toy problems with $K=256$ arms, using the stepsize $\eta_t = \sqrt{\log K / (K(t+1))}$ and $m=5$. The first problem is illustrated in ~\autoref{fig:EXP3 changing-cliff}, where a fixed fraction $20\%$ of the arms always returns $0$ and the rest return a loss of $1$. Moreover, the arms that return favorable loss changes throughout the running time, as the reader may observe from the $5$ `bumps' in the cumulative loss. In ~\autoref{fig:EXP3 stochastic} we simulate a `stochastic' setting, where to each arm a distribution is associated, and every pull of that arm returns an i.i.d. copy from that distribution. In our experiment arm $k \in \{0,1,\dots,K-1\}$ has distribution $\sim (k/K-0.3 U) \lor 0$ where $U \sim \unif(0,1)$. In particular, arm $0$ always returns $0$. 

In all our experiments, we implemented the data structure defined in ~\autoref{lem:data structure} using the SortedList class of the sortedcontainers Python library \cite{sortedcontainers}.

\begin{figure}[H]
\centering
\begin{minipage}{0.47\textwidth}
  \centering
  \includegraphics[width=\linewidth]{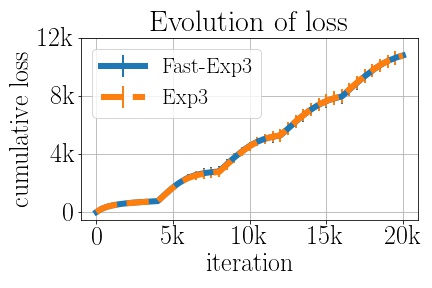}
  \caption{Error bars denote $4$ standard deviations over $20$ runs.}
  \label{fig:EXP3 changing-cliff}
\end{minipage}%
\hfill
\begin{minipage}{0.47\textwidth}
  \centering
  \includegraphics[width=\linewidth]{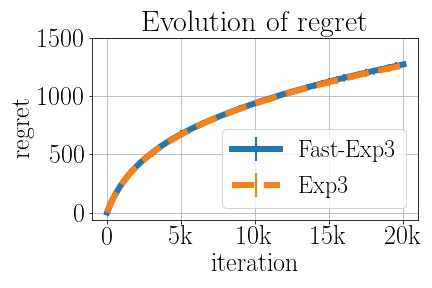}
  \caption{Error bars denote $4$ standard deviations over $20$ runs. }
  \label{fig:EXP3 stochastic}
\end{minipage}
\end{figure}

\section{Proofs of the complexity bounds}\label{scn:proof_overview}



We begin with a few general comments on the lower bounds.
Recall that the rejection sampling task, given query access to the unnormalized distribution $\tilde p$, is to construct an upper envelope $\tilde q \ge \tilde p$ satisfying $Z_q \le 2Z_p$.
We in fact prove lower bounds for an easier task, namely, the task of constructing a proposal distribution $q$ such that $\norm{p/q}_\infty \le 2$. Note that if we have an upper envelope $\tilde q \ge \tilde p$ with $Z_q \le 2Z_p$, then the corresponding normalized distribution $q$ satisfies
\begin{align}\label{eq:easier_task}
    \bigl\lVert \frac{p}{q} \bigr\rVert_\infty
    &:= \sup_{x\in \ms X} \frac{p(x)}{q(x)}
    = \sup_{x\in \ms X} \underbrace{\frac{\tilde p(x)}{\tilde q(x)}}_{\le 1} \underbrace{\frac{Z_q}{Z_p}}_{\le 2}
    \le 2\,,
\end{align}
so the latter task is indeed easier.

The proofs of the lower bounds are to an extent situational, but we outline here a fairly generic strategy that seems useful for many (but not all) classes of distributions. First, we fix a reference distribution $p^\star \in \mc P$ and assume that the algorithm has access to queries to an oracle for $p^\star$ (up to normalization). Also, suppose that the algorithm makes queries at the points $x_1,\dotsc,x_n$.
Since we assume that the algorithm is deterministic, if $p \in \mc P$ is another distribution which agrees with $p^\star$ at the queries $x_1,\dotsc,x_n$ (up to normalization), then the algorithm produces the same output regardless of whether it is run on $p$ or $p^\star$. In particular, the output $q$ of the algorithm must satisfy both $\norm{p/q}_\infty \le 2$ and $\norm{p^\star/q}_\infty \le 2$.

More generally, for each $y \in [N]$ we can construct an adversarial perturbation $p_y \in \mc P$ of $p^\star$ which maximizes the probability of $y$, subject to being consistent with the queried values.
Then the rejection sampling guarantee of the algorithm ensures that
\begin{align*}
    2
    &\ge \bigl\lVert \frac{p_y}{q} \bigr\rVert_\infty
    \ge \frac{p_y(y)}{q(y)}
    = \frac{1}{q(y)} \sup\Bigl\{p(y) : p \in \mc P, \; \frac{p(x_i)}{p(x_j)} = \frac{p^\star(x_i)}{p^\star(x_j)}~\text{for all}~i,j \in [n]\Bigr\}\,.
\end{align*}
Since $q$ is a probability distribution, this yields the inequality
\begin{align}\label{eq:lower_bd_ineq}
    1
    &= \sum_{y\in [N]} q(y)
    \ge \frac{1}{2} \sum_{y\in [N]}\sup\Bigl\{p(y) : p \in \mc P, \; \frac{p(x_i)}{p(x_j)} = \frac{p^\star(x_i)}{p^\star(x_j)}~\text{for all}~i,j \in [n]\Bigr\}\,.
\end{align}
By analyzing this inequality for the various classes of interest, it is seen to furnish a lower bound on the number of queries $n$. Thus, the lower bound strategy consists of choosing a judicious reference distribution $p^\star$, constructing the adversarial perturbations $p_y$, and using the inequality~\eqref{eq:lower_bd_ineq} to produce a lower bound on $n$.

\subsection{Monotone distributions}\label{scn:monotone_pf}

\begin{theorem}\label{thm:monotone}
Let $\mc P$ be the class of monotone distributions supported on $[N]$, as given in Definition~\ref{defn:monotone}. Then the rejection sampling complexity of $\mc P$ is $\Theta(\log N)$.
\end{theorem}

\subsubsection{Upper bound}\label{scn:monotone_upper}

In the proof, let $p$ denote the target distribution and assume that we can query the values of $\tilde p = p Z_p$.
Also, by rounding $N$ up to the nearest power of $2$, and considering $p$ to be supported on this larger alphabet, we can assume that $N$ is a power of $2$; this will not affect the complexity bound.

\begin{proof}
We construct the upper envelope $\tilde q$ as follows: first query the values of $\tilde p(2^i)$, $0\le i \le \log_2 N-1$, which requires $\mc O (\log N)$ queries; then $\tilde q$ is given as follows: set $\tilde q(1) := \tilde p(1)$ and
\begin{align*}
    \tilde q(x) := \tilde p(2^i) \, , \qquad \text{ for } x\in (2^i, 2^{i+1}]\,.
\end{align*}
Note that $\tilde q$ is an upper envelope of $\tilde p$ because $p$ is assumed to be monotone.

To complete the proof of the upper bound in Theorem~\ref{thm:monotone}, we just have to check that $Z_q/Z_p \le 2$. We use the definitions of the normalizing constants:
\begin{align*}
    Z_p
    &= \tilde p(1) + \sum_{i=0}^{\log_2 N - 1} \sum_{x=2^i+1}^{2^{i+1}} \tilde p(x)
    \ge \tilde p(1) + \sum_{i=0}^{\log_2 N - 1} 2^i \tilde p(2^{i+1}) \\
    &\ge \tilde p(1) + \frac{1}{2} \sum_{i=0}^{\log_2 N - 2} \sum_{x=2^{i+1}+1}^{2^{i+2}} \tilde q(x)
    = \underbrace{\tilde p(1)}_{= (\tilde q(1)+\tilde q(2))/2} + \frac{1}{2} \sum_{x=3}^N \tilde q(x)
    = \frac{1}{2} \sum_{x=1}^N \tilde q(x)
    = \frac{1}{2} \, Z_q\,.
\end{align*}
The bound above shows that $Z_q/Z_p\le 2$, which concludes the proof.
\end{proof}

\subsubsection{Lower bound}\label{scn:monotone_lower}

In this proof, we follow the lower bound strategy encapsulated in~\eqref{eq:lower_bd_ineq}.

\begin{proof}
Let $x_1< \dotsc < x_n$ denote the queries; to simplify the proof, we will also assume that $1$ and $N$ are part of the queries. This can be interpreted as giving the algorithm two free queries, and the rest of the proof can be understood as a lower bound on the number of queries that the algorithm made, minus two.

We choose our reference distribution to be $p^\star(x) \propto 1/x$, i.e., we take
\begin{align*}
    p^\star(x)
    &= \frac{c_N}{x\log N}\,, \qquad\text{for}~x \in [N]\,,
\end{align*}
where $c_N$ is used to normalize the distribution, and it satisfies $c_N \asymp 1$.
To construct the adversarial perturbation $p_y$, suppose that $y$ lies strictly between the queries $x_i$ and $x_{i+1}$. Let $\alpha := {(y - x_i)}^{-1} \sum_{x_i < x \le y} p^\star(x)$ denote the average of $p^\star$ on $(x_i, y]$.
Then, we define
\begin{align*}
    p_y(x) := \begin{cases}
        \alpha\,, & x_i < x \le y\,, \\
        p^\star(x)\,, & \text{otherwise}\,.
    \end{cases}
\end{align*}
Since we replace the part of $p^\star$ on $(x_i, y]$ with its average value on this interval, then $p_y$ is also a probability distribution:
\begin{align*}
    \sum_{x\in [N]} p_y(x)
    = \sum_{x \in [N]} p^\star(x) + \sum_{x_i < x \le y} \{\alpha - p^\star(x)\}
    = 1\,.
\end{align*}
Since $p^\star$ is decreasing, it is clear that $p_y$ is too.
Also, we can lower bound $\alpha$ via
\begin{align*}
    \alpha
    &= \frac{1}{y-x_i} \sum_{x_i < x \le y} \frac{c_N}{x \log N}
    \ge \frac{c_N}{\log N} \, \frac{\log \frac{y+1}{x_i+1}}{y-x_i}\,.
\end{align*}
Since $p_y$ agrees with the queries, we can substitute this into~\eqref{eq:lower_bd_ineq} to obtain
\begin{align*}
    \frac{2 \log N}{c_N}
    &
    \geq \sum_{i=1}^{n-1} \sum_{x_i < y < x_{i+1}} \frac{\log \frac{y+1}{x_i+1}}{y-x_i}\,.
\end{align*}
In what follows, let $\Delta_i := x_{i+1}/x_i$. We will only focus on the terms with $\Delta_i \ge 8$, so assume now that $\Delta_i \ge 8$.
Let us evaluate the inner term via dyadic summation: 
\begin{align*}
    \sum_{x_i < y < x_{i+1}} \frac{\log \frac{y+1}{x_i+1}}{y-x_i}
    &= \sum_{0 < y < x_{i+1}-x_i} \frac{\log(1+\frac{y}{x_i+1})}{y} \\
    &\geq \sum_{0 \le j \le \log_2\frac{x_{i+1}-x_i-1}{x_i+1}-1} \; \sum\limits_{2^j (x_i+1) \leq y < 2^{j+1} (x_i+1)} \frac{\log(1 + \frac{y}{x_i+1})}{y} \\
    &\gtrsim  \sum_{0 \le j \le \log_2\frac{x_{i+1}-x_i-1}{x_i+1}-1} \; \sum\limits_{2^j (x_i+1) \leq y < 2^{j+1} (x_i+1)} \frac{j}{2^{j+1} \, (x_i+1)} \\
    &\gtrsim  \sum_{0 \le j \le \log_2(\Delta_i/4)} j
    \gtrsim {(\log \Delta_i)}^2\,.
\end{align*}
Let $A := \{i\in [n-1] : \Delta_i \ge 8\}$.
Our calculations above yield
\begin{equation*}
   \log N \gtrsim \sum\limits_{i \in A} {(\log \Delta_i)}^2\,.
\end{equation*}
Observe now that $\prod_{i=1}^{n-1}\Delta_i = N$ and $\prod_{i\in A^\comp} \Delta_i \le 8^{\abs{A^\comp}} \le 8^n$, so that $\prod_{i\in A} \Delta_i \ge N/8^n$. Hence, applying the Cauchy-Schwarz inequality,
\begin{align*}
    \log N
    &\gtrsim \frac{1}{\abs A} \, \Bigl\lvert \sum_{i\in A} \log \Delta_i \Bigr\rvert^2
    \ge \frac{{[{(\log N - n\log 8)}_+]}^2}{\abs A}\,.
\end{align*}
We can now conclude as follows: either $n \ge (\log N)/(2\log 8)$, in which case we are done, or else $n \le (\log N)/(2\log 8)$. In the latter case, the above inequality can be rearranged to yield $n-1 \ge \abs A \gtrsim \log N$, which proves the desired statement in this case as well.
\end{proof}

\subsection{Strictly unimodal distributions}

\begin{theorem}\label{thm:unimodal}
Let $\mc P$ be the class of strictly unimodal distributions supported on $[N]$, as given in Definition~\ref{defn:unimodal}. Then the rejection sampling complexity of $\mc P$ is $\Theta(\log N)$.
\end{theorem}

\subsubsection{Upper bound}\label{scn:unimodal_upper}

\begin{proof}
    Since the strategy is very similar to the upper bound for the class of monotone distributions (\autoref{thm:monotone}), we briefly outline the procedure here.
    Using binary search, we can locate the mode of the distribution using $\mc O(\log N)$ queries. Once the mode is located, the strategy for constructing an upper envelope for monotone distributions can be employed on each side of the mode.
\end{proof}

\subsubsection{Lower bound}\label{scn:unimodal_lower}

\begin{proof}
    We again refer to the class of monotone distributions (\autoref{thm:monotone}), for which the lower bound is given in~\ref{scn:monotone_lower}. Essentially the same proof goes through for this setting as well, and we make two brief remarks on the modifications. First, the reference distribution $p^\star$ in that proof is also strictly unimodal. Second, although the adversarial perturbations $p_y$ constructed in that proof are not strictly unimodal, they can be made strictly unimodal via infinitesimal perturbations, so it is clear that the proof continues to hold.
\end{proof}

\subsection{Cliff-like distributions}\label{scn:cliff_appendix}
\begin{theorem}\label{thm:cliff}
Let $\mc P$ be the class of cliff-like distributions supported on $[N]$, as given in Definition~\ref{defn:cliff}. Then the rejection sampling complexity of $\mc P$ is $\Theta(\log\log N)$.
\end{theorem}

\subsubsection{Upper bound}\label{scn:cliff_upper}

\begin{proof}
    Since the class of cliff-like distributions is contained in the class of discrete log-concave distributions, the upper bound for the former class is subsumed by Theorem~\ref{thm:discrete_lc} on the latter class.
\end{proof}

\subsubsection{Lower bound}\label{scn:cliff_lower}

In this proof, we reduce the task of building a rejection sampling proposal $q$ for the class of cliff-like distributions to the computational task of finding the cliff in an array.
Formally, the latter task is defined as follows.

\begin{task}[finding the cliff in an array]\label{task:finding_cliff}
    There is an unknown array of the form \[ a = [1,\dotsc,1,0,\dotsc,0] \] of size $N$.
    Let $k$ be the largest index such that $a[i] = 1$.
    Given query access to the array, what is the minimum number of queries needed to determine the value of $k$?
\end{task}

The number of queries needed to solve~\autoref{task:finding_cliff} is $\Theta(\log N)$ (achieved via binary search).
We now give the reduction.

\begin{proof}
Suppose that the algorithm makes queries to $\tilde p$.
Let $x_-$ be the largest query point with $\tilde p(x_-) > 0$, and let $x_+$ be the smallest query point with $\tilde p(x_+) = 0$.
Given $x_- \le y < x_+$, the adversarial perturbation $p_y$ is the uniform distribution on $[y]$.
Substituting this into~\eqref{eq:lower_bd_ineq}, and replacing ratios between $p^\star$ with ratios between $\tilde p$, we obtain
\begin{align*}
    2
    &\ge \sum_{x_- \le y < x_+} p_y(y)
    = \sum_{x_- \le y < x_+} \frac{1}{y}
    \ge \log\frac{x_+}{x_-}\,.
\end{align*}
Hence, an algorithm which can achieve the desired rejection sampling guarantee can guarantee that $x_+ \le cx_-$, where $c = e^2$ is a constant.

This reduces the lower bound for the rejection sampling complexity to the following question: what is the minimum number of queries to ensure that $x_+ \le cx_-$?

At this point we can reduce to~\autoref{task:finding_cliff}.
Suppose after $n$ queries we can indeed ensure that $x_+ \le cx_-$.
Consider an array $a$ of size $\log_c N$, which has a cliff at index $k$. (We may round $c$ up to the nearest integer, and $N$ up to the nearest multiple of $c$ in order to avoid ceilings and floors.)
From this array we construct the unnormalized distribution $\tilde p$ on $[N]$ via
\begin{align*}
    \tilde p(x) := \one\{x \le c^k\}\,, \qquad x \in [N]\,.
\end{align*}
The rejection sampling algorithm provides us with $x_+ \le cx_-$ such that $\tilde p(x_-) = 1$ and $\tilde p(x_+) = 0$, i.e., $x_- \le c^k < x_+ \le cx_-$.
Taking logarithms, we see that
\begin{align*}
    \log_c x_-
    &\le k
    < \log_c x_- + 1\,.
\end{align*}
Hence, taking $\log_c x_-$ and rounding to the nearest integer (possibly doing a constant number of extra queries to the array afterwards for verification) locates the cliff $k$ in $n$ queries.
Using the lower bound for~\autoref{task:finding_cliff}, we see that $n = \Omega(\log \log N)$ as claimed.
\end{proof}

\subsection{Discrete log-concave distributions}
\begin{theorem}\label{thm:discrete_lc}
    Let $\mc P$ be the class of discrete log-concave distributions on $[N]$, as in Definition~\ref{defn:lc}, and recall that the modes of the distributions are assumed to be $1$. Then the rejection sampling complexity of $\mc P$ is $\Theta(\log\log N)$.
\end{theorem}

\subsubsection{Upper bound}\label{scn:lc_upper}
We make a few simplifying assumptions just as in the upper bound proof for Theorem~\ref{thm:monotone}. Let $p$ denote the target distribution, assume that the queries are made to $\tilde p = pZ_p$, and let $V:\R \to \R \cup \{\infty\}$ be a convex function such that $\tilde p(x) = \exp( - V(x))$ for $x \in [N]$. Also, we round $N$ up to the nearest power of $2$, which does not change the complexity bound.
\begin{proof}
    First we make one query to obtain the value of $\tilde p(1)$. Then we find the integer $0\le i_0\le \log_2 N - 1$ (if it exists) such that 
    \begin{align*}
        2\tilde p(2^{i_0}) \ge \tilde p(1)\,, \qquad  2\tilde p(2^{i_0+1}) \le \tilde p(1)\, .
    \end{align*}
    To do this, observe that the values $\tilde p(2^i), \; 0 \le i \le \log_2 N$ are decreasing, and by performing binary search over these $\mc O(\log N)$ values we can find the integer $i_0$ or else conclude that it does not exist using $\mc O(\log \log N)$ queries.
    
    If $i_0$ does not exist, then the target satisfies $2 \tilde p(x) \ge \tilde p(1)$ for all $x\in[N]$, so the constant upper envelope $\tilde q= \tilde p(1)$ suffices.
    
    If $i_0$ exists, denote $x_0 = 2^{i_0 + 1}$, and construct the upper envelope $\tilde q$ as follows: query $\tilde p(x_0)$, and let
    \begin{align*}
        \tilde q(x) &= \begin{cases}
        \tilde p(1)\,, & x < x_0\,,\\
        \tilde p(x_0)\, e^{-\lambda \,(x-x_0)}\,, & x\ge x_0\, ,
        \end{cases}\\
        \lambda &= \frac{\log \frac{\tilde p(1)}{\tilde p(x_0)}}{x_0-1} = \frac{V(x_0) - V(1)}{x_0-1}\,.
    \end{align*}
    We check that $\tilde q$ is a valid upper envelope of $\tilde p$. If we take logarithms and denote $V_q(x) = -\log \tilde q(x)$, then we see that
    \begin{align*}
        V_q(x) = \begin{cases}
        V(1)\,, & x < x_0\,,\\
        V(x_0) + \lambda \,(x-x_0)\,, & x \ge x_0\, .
        \end{cases}
    \end{align*}
    Because $V$ is convex, we see that $V_q$ is a lower bound of $V$, so $\tilde q$ is an upper bound of $\tilde p$.
    
    To finish the proof, we just have to bound $Z_q/Z_p$. Let $Z_{q,1} = \sum_{x< x_0} \tilde q(x)$, and $Z_{q,2} = \sum_{x\ge x_0} \tilde q(x)$, so $Z_q = Z_{q,1} + Z_{q,2}$. We will bound these two terms separately. For the first term, by the definition of $x_0$ we can bound
    \begin{align*}
        Z_{q,1}
        = \sum_{x<x_0} \tilde p(1)
        &\le 2 \sum_{x < x_0/2} \tilde p(1)
        \le 4 \sum_{x < x_0/2} \tilde p(x) \,.
    \end{align*}
    For the second term,
    \begin{align*}
        Z_{q,2} &\le \tilde p(x_0) \sum_{z=0}^\infty e^{-\lambda z}
        = \tilde p(x_0) \, {(1 - e^{-\lambda})}^{-1}
        = \tilde p(x_0) \, \Bigl(1 - \bigl(\frac{\tilde p(x_0)}{\tilde p(1)}\bigr)^{x_0-1}\Bigr)^{-1}
        \le 2 \tilde p(x_0)\,.
    \end{align*}
    Putting this together,
    \begin{align*}
        Z_q
        &= Z_{q,1} + Z_{q,2}
        \le 4Z_p\,.
    \end{align*}
    For clarity, we have presented the proof with the bound $Z_q/Z_p \le 4$. At the cost of more cumbersome proof, the above strategy can be modified to yield the guarantee $Z_q/Z_p \le 2$.
\end{proof}

\subsubsection{Lower bound}\label{scn:lc_lower}

\begin{proof}
    Since the class of cliff-like distributions is contained in the class of discrete log-concave distributions, the lower bound for the latter class is subsumed by Theorem~\ref{thm:cliff} on the former class.
\end{proof}

\subsection{Monotone on a binary tree}
\begin{theorem}\label{thm:tree}
Let $\mc P$ be the class of monotone distributions on a binary tree with $N$ vertices, as in \autoref{defn:tree}. Then the rejection sampling complexity of $\mc P$ is $\Theta (N/(\log N))$.
\end{theorem}

Let $\eu T$ denote the binary tree.
For the upper bound, we may embed $\eu T$ into a slightly larger tree, and for the lower bound we can perform the construction on a slightly smaller tree.
In this way, we may assume that $\eu T$ is a complete binary tree of depth $\ell$, and hence $N = \sum_{j=0}^\ell 2^j = 2^{\ell+1}-1$; this does not affect the complexity results. Throughout the proofs, we write $\abs x$ for the depth of the vertex $x$ in the tree, where the root is considered to be at depth $0$.

\subsubsection{Upper bound}\label{scn:tree_upper}

\begin{proof}
Let $c$ be a constant to be chosen later.
The algorithm is to query the value of $\tilde p$ at all vertices at depth at most $\ell_0 := \ell - \log_2 \ell + c$.
Then the upper envelope is constructed as follows,
\begin{align*}
    \tilde q(x)
    &:= \begin{cases} \tilde p(x)\,, & \text{if}~\abs x \le \ell_0\,, \\ \tilde p(y)\,, & \text{if}~\abs x > \ell_0\,,\; \abs y = \ell_0\,,\; \text{and}~x~\text{is a descendant of}~y\,. \end{cases}
\end{align*}
Clearly $\tilde q\ge \tilde p$.
Also, the number of queries we made is
\begin{align*}
    \sum_{j=0}^{\ell_0} 2^j
    &= 2^{\ell_0+1} - 1
    \lesssim \frac{2^\ell}{\ell}
    \lesssim \frac{N}{\log N}\,.
\end{align*}
Finally, we bound the ratio $Z_q/Z_p$. By definition,
\begin{align*}
    Z_q =\sum_{x\in \eu T} \tilde q(x)
    &= \sum_{x \in \eu T, \; \abs x \le \ell_0} \tilde p(x) + \sum_{x \in \eu T, \; \abs x > \ell_0} \tilde q(x)\,.
\end{align*}
For the second sum, we can write
\begin{align*}
    \sum_{x\in\eu T, \; \abs x > \ell_0} \tilde q(x)
    &= \sum_{y\in \eu T, \; \abs y = \ell_0} \tilde p(y) \, (2^{\ell-\ell_0+1} - 1)
    = \sum_{y\in \eu T, \; \abs y = \ell_0} \tilde p(y) \, (2^{\log_2\ell-c +1} - 1) \\
    &\le \ell \, 2^{-c+1} \sum_{y\in \eu T, \; \abs y = \ell_0} \tilde p(y)\,.
\end{align*}
On the other hand, if $x$ denotes any vertex, let $x_1$, $x_2$ denote its two children; then, for any level $j$,
\begin{align*}
    \sum_{x\in\eu T, \; \abs x = j+1} \tilde p(x)
    &= \sum_{x\in\eu T, \; \abs x = j} \{\tilde p(x_1) + \tilde p(x_2)\}
    \le \sum_{x\in\eu T, \; \abs x = j} \tilde p(x)\,.
\end{align*}
Hence,
\begin{align*}
    \sum_{x\in\eu T, \; \abs x \le \ell_0} \tilde p(x)
    &\ge (\ell_0 +1) \sum_{x\in\eu T, \; \abs x = \ell_0} \tilde p(x)
\end{align*}
which yields
\begin{align*}
    Z_q
    &\le \bigl(1 + \frac{\ell \, 2^{-c+1}}{\ell_0+1}\bigr) \sum_{x\in\eu T, \; \abs x \le \ell_0} \tilde p(x)
    \le 2 Z_p\,,
\end{align*}
if $\ell$ and $c$ are sufficiently large.
\end{proof}

\subsubsection{Lower bound}\label{scn:tree_lower}

The proof of the lower bound follows the strategy encapsulated in \eqref{eq:lower_bd_ineq}.
\begin{proof}
Suppose that an algorithm achieves rejection sampling ratio $2$ with $n$ queries.
Again let $\ell_0 := \ell-\log_2\ell + c$, where the constant $c$ will possibly be different from the one in the upper bound.
The reference distribution will be
\begin{align*}
    \tilde p(x)
    &:= \begin{cases} 2^{-\abs x}\,, & \abs x \le \ell_0\,, \\ 0\,, & \abs x > \ell_0\,. \end{cases}
\end{align*}
Note that $p\in \mc P$.
The normalizing constant is $Z_{p} = \ell_0 + 1$, since there are $2^j$ vertices at level $j$.
For each $\abs y > \ell_0$, we will create a perturbation distribution $p_y$ in the following way:
\begin{align*}
    \tilde p_y(x)
    := \begin{cases}2^{-\abs x} \,, & \abs x \le \ell_0\,, \\ 2^{-\ell_0}\,, & \abs x > \ell_0~\text{and}~y~\text{is a descendant of}~x~(\text{or equal to}~x)\,,\\
    0, & \text{otherwise}\,.\end{cases}
\end{align*}
Thus, $\tilde p_y$ places extra mass on the path leading to $y$; note also that $p_y \in \mc P$.
The normalizing constant for $p_y$ is
\begin{align*}
    Z_{p_y}
    &= Z_{p} + \sum_{j=\ell_0+1}^{\abs y} 2^{-\ell_0}
    \le \ell_0 + 1 + (\ell - \ell_0)\, 2^{-\ell_0}
    = \ell_0\, \{1+o(1)\}\,,
\end{align*}
where $o(1)$ tends to $0$ as $\ell\to\infty$.

Next, let $\mc Q$ denote the set of vertices $x$ at level $\ell_0$ for which at least one of the descendants of $x$ (not including $x$ itself) is queried by the algorithm, and let $\mc Q^\comp$ denote the vertices at level $\ell_0$ which do not belong to $\mc Q$.
Note if $x \in \mc Q^\comp$ and $y$ is a descendant of $x$, then $p_y$ is consistent with the queries made by the algorithm.
Let $\mc D(x)$ denote the descendants of $x$.
Now, applying~\eqref{eq:lower_bd_ineq} with $p^\star=p$,
\begin{align*}
    2
    &\ge \sum_{x \in \eu T, \; \abs x \le \ell_0} p(x) + \sum_{x\in \mc Q^\comp} \sum_{y\in \mc D(x)} p_y(y)
    = 1 + \sum_{x\in \mc Q^\comp} \sum_{y\in \mc D(x)} p_y(y)
\end{align*}
which yields
\begin{align*}
    1
    &\ge \sum_{x\in \mc Q^\comp} \sum_{y\in \mc D(x)} p_y(y)
    \ge \sum_{x\in \mc Q^\comp} \frac{2^{-\ell_0}}{\ell_0 \, (1+o(1))} \, (2^{\ell-\ell_0+1} - 2)
    \gtrsim \frac{2^{\ell-2\ell_0+1}}{\ell_0 \, (1+o(1))} \, \{2^{\ell_0} - \abs{\mc Q}\}\,.
\end{align*}
It then yields
\begin{align*}
    n
    &\ge \abs{\mc Q}
    \gtrsim 2^{\ell_0} - \frac{\ell_0 \, (1+o(1))}{2^{\ell-2\ell_0+1}}
    = 2^{\ell_0} \, \Bigl( 1 - \frac{\ell_0 \, (1+o(1))}{2^{\ell-\ell_0+1}}\Bigr) \\
    &= 2^{\ell_0} \, \Bigl( 1 - \frac{\ell_0 \, (1+o(1))}{2^{\log_2 \ell-c+1}}\Bigr)
    = 2^{\ell_0} \, \Bigl( 1 - \frac{\ell_0 \, (1+o(1))}{\ell \, 2^{-c+1}}\Bigr)\,.
\end{align*}
If we now choose $c \ll 0$ to be a \emph{negative} constant, we can verify
\begin{align*}
    n
    &\gtrsim 2^{\ell_0}
    = 2^{\ell - \log_2\ell +c}
    \gtrsim \frac{N}{\log N}\,,
\end{align*}
completing the proof.
\end{proof}

\subsubsection{An alternate definition of monotone}

In this section, we show that if we adopt an alternative definition of monotone on a binary tree, then the rejection sampling complexity is trivial.

\begin{theorem}\label{thm:alternate_monotone}
Let $\mc P$ be the class of probability distributions $p$ on a binary tree with $N$ vertices, with maximum depth $\lceil \log_2(N+1)\rceil$, such that if for every non-leaf vertex $x$, if the children of $x$ are $x_1$ and $x_2$, then $p(x) \ge p(x_1) \vee p(x_2)$. Then, the rejection sampling complexity of $\mc P$ is $\Theta(N)$.
\end{theorem}
\begin{proof}
    It suffices to show the lower bound, and the proof will be similar to the one in Appendix~\ref{scn:tree_lower}. We may assume that the binary tree is a complete binary tree with depth $\ell$.
    Suppose that an algorithm achieves a rejection sampling ratio $2$ after $n$ queries. We define the reference distribution $p^\star$ via
    \begin{align*}
        \tilde p^\star(x)
        &:= \begin{cases} 1\,, & \abs x \le \ell - 2\,, \\ 0\,, & \abs x > \ell - 2\,. \end{cases}
    \end{align*}
    The normalizing constant is $Z_{p^\star} = 2^{\ell-1}-1$. For each leaf vertex $y$, we define the perturbation distribution $p_y$ via
    \begin{align*}
        \tilde p_y(x)
        &:= \begin{cases} 1\,, & \abs x \le \ell - 2~\text{or}~x~\text{is an ancestor of}~y \; (\text{including if}~x=y)\,, \\ 0\,, & \abs x > \ell - 2\,. \end{cases}
    \end{align*}
    The normalizing constant of $p_y$ is $Z_{p_y} = 2^{\ell-1}+1$.
    
    Let $\mc Q$ denote the set of leaf vertices which are queried by the algorithm, and let $\mc Q^\comp$ denote the set of leaf vertices not in $\mc Q$. Then, from~\eqref{eq:lower_bd_ineq},
    \begin{align*}
        2
        &\ge \sum_{x\in\eu T, \; \abs x \le \ell-2} p^\star(x) + \sum_{y \in \mc Q^\comp} p_y(y)
        = 1 + \sum_{y \in \mc Q^\comp} p_y(y)
    \end{align*}
    and rearranging this yields
    \begin{align*}
        1
        &\ge \sum_{y \in \mc Q^\comp} \frac{1}{2^{\ell-1}+1}
        = \frac{1}{2^{\ell-1}+1} \, \{2^\ell - \abs{\mc Q}\}\,.
    \end{align*}
    This is further rearranged to yield
    \begin{align*}
        n
        &\ge \abs{\mc Q}
        \ge 2^{\ell-1} \, \bigl(2 - 1 - \frac{1}{2^{\ell-1}}\bigr)
        \gtrsim 2^\ell = N\,,
    \end{align*}
    where the last inequality holds if $\ell > 1$.
\end{proof}

\end{document}